%% file: main.tex
\documentclass{article} \usepackage{fullpage,times}
\usepackage{parskip} 

\usepackage{amsthm}

\usepackage[english]{babel}
\usepackage[utf8x]{inputenc}
\usepackage[T1]{fontenc}
\usepackage{amsmath}
\usepackage{amssymb}
\usepackage{amsfonts}
\usepackage{multirow}
\usepackage{subcaption}

\usepackage{bbm}
\usepackage{dsfont}
\usepackage{graphicx}
\usepackage{natbib}
\usepackage{cases}
\usepackage{color}
\usepackage[colorinlistoftodos]{todonotes}
\usepackage[colorlinks=true, allcolors=blue]{hyperref}
\usepackage{algorithm}
\usepackage[noend]{algpseudocode}
\usepackage{array}

\def\shownotes{0} 
\ifnum\shownotes=1
\newcommand{\authnote}[2]{{[#1: #2]}}
\else 
\newcommand{\authnote}[2]{{}}
\fi

\newtheorem{theorem}{Theorem}[section]

\newtheorem{proposition}[theorem]{Proposition}
\newtheorem{lemma}[theorem]{Lemma}
\newtheorem{corollary}[theorem]{Corollary}

\newtheorem{conjecture}[theorem]{Conjecture}

\newtheorem{definition}[theorem]{Definition}

\newtheorem{example}[theorem]{Example}

\newtheorem{assumption}[theorem]{Assumption}

\newtheorem{remark}[theorem]{Remark}

\numberwithin{equation}{section}

\title{Why Do Local Methods Solve Nonconvex Problems?}
\author{Tengyu Ma\\
Stanford University}

\input{macro/theorems}

\input{macro/general}

\input{macro/dict}
\input{macro/topic}

\input{macro/sid}

\begin{document}

\maketitle

\renewcommand{\cite}{\citet}

\begin{abstract}
	\normalsize{
	Non-convex optimization is ubiquitous in modern machine learning. Researchers devise non-convex objective functions and optimize them using off-the-shelf optimizers such as stochastic gradient descent and its variants, which leverage the local geometry and update iteratively. Even though solving non-convex functions is NP-hard in the worst case, the optimization quality in practice is often not an issue---optimizers are largely believed to find approximate global minima. Researchers hypothesize a unified explanation for this intriguing phenomenon: most of the local minima of the practically-used objectives are approximately global minima. We rigorously formalize it for concrete instances of machine learning problems.\footnote{This is the Chapter 21 of the book Beyond the Worst-Case Analysis of Algorithms~\citep{roughgarden2020beyond}.}
}
\end{abstract}

\tableofcontents

\newcommand{\cost}{\text{cost}}

\newpage

\input{landscape}

\input{isotonic}

\input{pca}

\input{mc}

\input{tensor}

\input{outlook}

\input{notes}

\bibliographystyle{plainnat}
\bibliography{refs/all}

\end{document}

%% file: macro/theorems.tex
\newtheorem{innercustomthm}{Claim}
\newenvironment{customthm}[1]
{\renewcommand\theinnercustomthm{#1}\innercustomthm}
{\endinnercustomthm}

%% file: macro/general.tex
\newcommand{\Po}{P_{\Omega}}
\renewcommand{\ni}{\noindent}

\newcommand{\argmax}{\textup{argmax}}
\newcommand{\argmin}{\textup{argmin}}

\newcommand{\Set}[1]{{\left\{#1\right\}}}
\newcommand{\hessian}{\textup{Hess }}
\newcommand{\grad}{\textup{grad }}
\newcommand{\norm}[1]{\lVert#1\rVert}
\newcommand{\Norm}[1]{\left\lVert#1\right\rVert}
\newcommand{\Id}{I}

\newcommand{\mper}{\,.}
\newcommand{\mcom}{\,,}
\newcommand{\R}{\mathbb R}

\numberwithin{equation}{section}

\newcommand{\E}{\mathbb{E}}
\newcommand{\Exp}{\mathop{\mathbb E}\displaylimits}

\newcommand{\eps}{\epsilon}

\newcommand{\diag}{\mathop{\mathrm{diag}}}

\ifx\BlackBox\undefined
\newcommand{\BlackBox}{\rule{1.5ex}{1.5ex}}  
\fi

\ifx\QED\undefined
\def\QED{~\rule[-1pt]{5pt}{5pt}\par\medskip}
\fi

\ifx\proof\undefined
\newenvironment{proof}{\par\noindent{\bf Proof\ }}{\hfill\BlackBox\\[2mm]}
\fi

\ifx\theorem\undefined
\newtheorem{theorem}{Theorem}

\newtheorem{lemma}[theorem]{Lemma}

\newtheorem{definition}[theorem]{Definition}

\fi

\ifx\axiom\undefined
\newtheorem{axiom}[theorem]{Axiom}
\fi

\newcommand{\inner}[1]{\langle #1\rangle}

%% file: macro/dict.tex
\newcommand\poly{\mbox{poly}}

%% file: macro/topic.tex
\newcommand{\xstar}{x^*}

%% file: macro/sid.tex
\newcommand\polylog{\operatorname{polylog}}

\newcommand{\sphere}{S}

\newcommand{\cM}{\mathcal{M}}
\newcommand{\tangent}{T}

\newcommand{\CONDNAME}{{weakly-quasi-convex}\xspace}

\newcommand{\ball}{\mathcal{B}}

\let\epsilon=\varepsilon

\numberwithin{equation}{section}

\newcommand\MYcurrentlabel{xxx}

\newcommand{\MYstore}[2]{%
	\global\expandafter \def \csname MYMEMORY #1 \endcsname{#2}%
}

\newcommand{\MYload}[1]{%
	\csname MYMEMORY #1 \endcsname%
}

\newcommand{\MYnewlabel}[1]{%
	\renewcommand\MYcurrentlabel{#1}%
	\MYoldlabel{#1}%
}

\newcommand{\MYdummylabel}[1]{}

\newcommand{\torestate}[1]{%
	\let\MYoldlabel\label%
	\let\label\MYnewlabel%
	#1%
	\MYstore{\MYcurrentlabel}{#1}%
	\let\label\MYoldlabel%
}

\newcommand{\restatetheorem}[1]{%
	\let\MYoldlabel\label
	\let\label\MYdummylabel
	\begin{theorem*}[Restatement of \prettyref{#1}]
		\MYload{#1}
	\end{theorem*}
	\let\label\MYoldlabel
}

\newcommand{\restatelemma}[1]{%
	\let\MYoldlabel\label
	\let\label\MYdummylabel
	\begin{lemma*}[Restatement of \prettyref{#1}]
		\MYload{#1}
	\end{lemma*}
	\let\label\MYoldlabel
}

\newcommand{\restateprop}[1]{%
	\let\MYoldlabel\label
	\let\label\MYdummylabel
	\begin{proposition*}[Restatement of \prettyref{#1}]
		\MYload{#1}
	\end{proposition*}
	\let\label\MYoldlabel
}

\newcommand{\restatefact}[1]{%
	\let\MYoldlabel\label
	\let\label\MYdummylabel
	\begin{fact*}[Restatement of \prettyref{#1}]
		\MYload{#1}
	\end{fact*}
	\let\label\MYoldlabel
}

\newcommand{\restate}[1]{%
	\let\MYoldlabel\label
	\let\label\MYdummylabel
	\MYload{#1}
	\let\label\MYoldlabel
}

%% file: landscape.tex
\section{Introduction}

Optimizing non-convex functions has become the standard algorithmic technique in modern machine learning and
artificial intelligence. It is increasingly important to understand the working of the existing heuristics for optimizing non-convex functions, so that we can design more efficient optimizers with guarantees. The worst-case intractability result says that finding a global minimizer of a non-convex optimization problem --- or even just a degree-4 polynomial --- is NP-hard. Therefore, theoretical analysis with global guarantees has to depend on the special properties of the target functions that we optimize. To characterize the properties of the real-world objective functions, researchers have hypothesized that many objective functions for machine learning problems have the property that 
\begin{equation}\label{eqn:11}
	 \textup{all or most local minima are approximately global minima.}
\end{equation}
Optimizers based on local derivatives can solve this family of functions in polynomial time (under some additional technical assumptions that will discussed below). Empirical evidences also suggest practical objective functions from machine learning and deep learning may have such a property. In this chapter, we formally state the algorithmic result that local methods can solve objective with property~\eqref{eqn:11} in Section~\ref{sec:landscape}, and then rigorously prove that this  property holds for a few objectives arising from several key machine learning problems: generalized linear models (Section~\ref{sec:glm}), principal component analysis (Section~\ref{sec:pca}), matrix completion (Section~\ref{sec:matrix}), and tensor decompositions (Section~\ref{sec:tensor}). We will also briefly touch on recent works on neural networks (Section~\ref{sec:neural_net}). 

\section{Analysis Technique: Characterization of the Landscape}\label{sec:landscape}
In this section, we will show that a technical and stronger version of the property~\eqref{eqn:11} implies that many optimizers can converge to a global minimum of the objective function. \subsection{Convergence to a local minimum}

We consider a objective function $f$, which is assumed to be twice-differentiable from $\R^d$ to $\R$. Recall that $x$ is a \textit{local minimum} of $f(\cdot)$ if there exists an open neighborhood $N$ of $x$ in which the function value is at least $f(x)$: $\forall z\in N, f(z)\ge f(x)$. A point $x$ is a \textit{stationary point} if it satisfies $\nabla f(x) = 0$. A \textit{saddle point} is a stationary point that is not a local minimum or maximum. 
We use $\nabla f(x)$ to denote the gradient of the function, and $\nabla^2 f(x)$ to denote the Hessian of the function ($\nabla^2 f(x)$ is an $d\times d$ matrix where $[\nabla^2 f(x)]_{i,j} = \frac{\partial^2}{\partial x_i \partial x_j} f(x)$). 
A local minimum $x$ must satisfy the first order necessary condition for optimality, that is, $\nabla f(x) = 0$, and the second order necessary condition for optimality, that is, $\nabla^2 f(x) \succeq 0$. (Here $A\succeq 0$ denotes that $A$ is a positive semi-definite matrix.)
Thus, A local minimum is a stationary point, so is a global minimum. 

However, $\nabla f(x) = 0$ and $\nabla^2 f(x)\succeq 0$ is not a sufficient condition for being a local minimum. For example, the original is not a local minimum of the function $f(x_1,x_2) = x_1^2 + x_2^3$ even though $\nabla f(0) =0$ and $\nabla^3 f(0) \succeq 0$. Generally speaking, along those direction $v$ where the Hessian vanishes (that is, $v^\top \nabla^2 f(x)v = 0$), the higher-order derivatives start to matter to the local optimality. In fact, finding a local minimum of a function is NP-hard~\citep{HillarL13}. 

Fortunately, with the following strict-saddle assumption, we can efficiently find a local minimum of the function $f$.  A strict-saddle function satisfies that every saddle point must have a strictly negative curvature in some direction. It assumes away the difficult situation in the example above where higher-order derivatives are needed to decide if a point is a local minimum.

\begin{definition}\label{def:strictsaddle}
 For $\alpha, \beta , \gamma \ge 0$, we say $f$ is $(\alpha,\beta,\gamma)$-\textit{strict saddle} if every $x\in \R^d$ satisfies \textit{at least} one of the following three conditions: \\
	1. $\norm{\nabla f(x)}_2 \ge \alpha$. \\
	2. $\lambda_{\min}(\nabla^2 f) \le -\beta$. \\
	3. There exists a local minimum $x^{\star}$ that is $\gamma$-close to $x$ in Euclidean distance.
\end{definition}

\begin{figure}
	\centering
	\begin{minipage}{.45\textwidth}
		\centering
		\includegraphics[height=4cm]{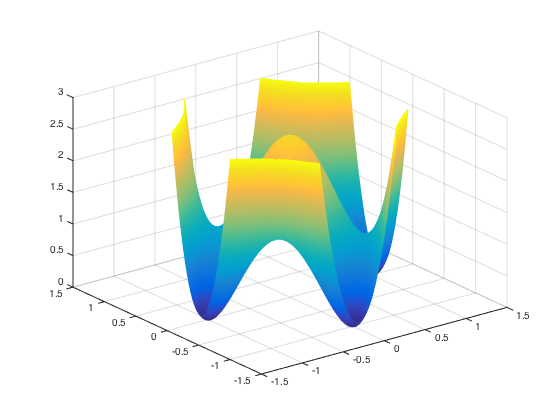}
	\end{minipage}
	~~~~
	\begin{minipage}{.5\textwidth}
		\centering
		\captionof{figure}{A two-dimensional function with the property that all local minima are global minima. It also satisfies the strict-saddle condition because all the saddle points have a trictly negative curvature in some direction. }
		\label{fig:wordvectors}
	\end{minipage}	
\end{figure}

This condition is conjectured to hold for many real-world functions, and will be proved to hold for various problems concretely. However, in general, verifying it mathematically or empirically  may be difficult. Under this condition, many algorithms can converge to a local minimum of $f$ in polynomial time as stated below.\footnote{Note that in this chapter, we only require polynomial time algorithm to be polynomial in $1/\eps$ when $\eps$ is the error. This makes sense for the downstream machine learning applications because very high accuracy solutions are not necessary due to intrinsic statistical errors.}

\begin{theorem}\label{thm:localmin}
Suppose $f$ is a twice differentiable $(\alpha,\beta,\gamma)$-strict saddle function from $\R^d\rightarrow \R$. 
Then, various optimization algorithms (such as stochastic gradient descent) can converge to a local minimum with $\epsilon$ error in Euclidean distance in time $\mbox{poly}(d, 1/\alpha, 1/\beta, 1/\gamma, 1/\epsilon)$.   
\end{theorem}

\subsection{Local optimality vs global optimality }

If a function $f$ satisfies the property that ``all local minima are global'' and the strict saddle property, we can provably find one of its global minima. (See Figure~\ref{fig:wordvectors} for an example of functions with this property. )

\begin{theorem}\label{thm:global}Suppose $f$ satisfies ``all local minima are global'' and the strict saddle property in a sense that all points satisfying approximately the necessary first order and second order optimality condition should be close to a global minimum: 
	
there exist $\epsilon_0, \tau_0 > 0$ and a universal constant $c> 0$ such that if a point  $x$ satisfies $\norm{\nabla f(x)}_2\le \epsilon\le \epsilon_0$ and $\nabla^2 f(x)\succeq -\tau_0\cdot \Id$, then $x$ is $\epsilon^c$-close to a global minimum of $f$.

Then, many optimization algorithms (including stochastic gradient descent and cubic regularization) can find a global minimum of $f$ up to $\delta$ error in $\ell_2$ norm in domain in time $\poly(1/\delta,1/\tau_0, d)$.
\end{theorem}
The technical condition of the theorem is often succinctly referred to as ``all local minima are global'', but its precise form, which is a combination of ``all local minima are global'' and the strict saddle condition,  is crucial. There are functions that satisfy ``all local minima are global'' but cannot be optimized efficiently. Ignoring the strict saddle condition may lead to misleadingly strong statements.

The condition of Theorem~\ref{thm:global} can be replaced by stronger ones which may occasionally be easier to verify, if they are indeed true for the functions of interests.  One of such conditions is that ``any stationary point is a global minimum." The gradient descent is known to converge to a global minimum linearly, as stated below. However, because this condition effectively rules out the existence of multiple disconnected local minima, it can't hold for many objective functions related to neural networks, which guarantees to have multiple local minima and stationary points due to a certain symmetry. 

\begin{theorem}\label{thm:quasi}
	Suppose a function $f$ has $L$-Lipschitz continuous gradients and satisfies the Polyak-Lojasiewicz condition: $\exists$ $\mu > 0$ and $\xstar$ such that for every $x$, 
	\begin{align}
	\Norm{\nabla f(x)}_2^2 \ge \mu(f(x)-f(\xstar))\ge 0 \mper\label{eqn:pl}
	\end{align}
	Then, the errors of the gradient descent with step size less than $1/(2L)$ decays geometrically. 
\end{theorem} 

 It can be challenging to verify  the Polyak-Lojasiewicz condition because the quantity $\Norm{\nabla{f(x)}}_2^2$ is often a complex function of $x$. An easier-to-verify but stronger condition is the quasi-convexity. Intuitively speaking, quasi-convexity says that at any point~$x$ the gradient should be negatively correlated with the direction~$x^*-x$ pointing towards the optimum.

\renewcommand{\CONDNAME}{{weakly-quasi-convex}}

\begin{definition}[Weak quasi-convexity]\label{def:cond}
	We say an objective function $f$ is \textit{$\tau$-\CONDNAME} 
	over a domain~$\ball$ with respect to the global minimum $x^*$ if there is a
	positive constant $\tau>0$ such that for all $x\in \ball$, 
	\begin{equation}
	\nabla f(x)^{\top} (x-x^*) \ge \tau (f(x)-f(x^*))\,.\label{eqn:condition}
	\end{equation}
The following one is another related condition, which is sometimes referred to as the restricted secant inequality (RSI):
		\begin{equation}
	\nabla f(x)^{\top} (x-x^*) \ge \tau \|x-x^*\|_2^2.\label{eqn:msi}
	\end{equation}	
\end{definition}
We note that convex functions satisfy~\eqref{eqn:condition} with $\tau =1$. Condition~\eqref{eqn:msi} is stronger than~\eqref{eqn:condition} because for smooth function, we have $\|x-x^*\|_2^2 \ge L (f(x) - f(x^*))$ for some constant $L$.\footnote{Readers who are familiar with convex optimization may realize that condition~\eqref{eqn:msi} is an extension of the strong convexity.} Conditions~\eqref{eqn:pl},~\eqref{eqn:condition}, and~\eqref{eqn:msi} all imply that all stationary points are global minimum because $\nabla f(x) = 0$ implies that $f(x) = f(x^*)$ or $x = x^*$.

\subsection{Landscape for manifold-constrained optimization}\label{sec:manifold}

We can extend many of the results in the previous section to the setting of constrained optimization over a smooth manifold. This section is only useful for problems in Section~\ref{sec:tensor} and casual readers can feel free to skip it. 

Let $\cM$ be a Riemannian  manifold.  Let $\tangent_x \cM$ be the tangent space to $\cM$ at $x$, and let $P_x$ be the projection operator to the tangent space $\tangent_x \cM$. Let  $\grad f(x)\in \tangent_x \cM$ be the gradient of $f$ at $x$ on $\cM$ and $\hessian f(x)$ be the Riemannian Hessian. Note that $\hessian f(x)$ is a linear mapping from $\tangent_x \cM$ onto itself. 
\begin{theorem}[Informally stated]\label{thm:manifold}
	\sloppy Consider the constrained optimization problem 
$
	\min_{x\sim \cM} f(x)
$. 
	Under proper regularity conditions, Theorem~\ref{thm:localmin} and Theorem~\ref{thm:global} still hold when replacing $\nabla f$ and $\nabla^2 f$ by $\grad f$ and $\hessian f$, respectively.
\end{theorem}
\noindent\textit{Backgrounds on manifold gradient and Hessian.} Later in Section~\ref{sec:tensor}, the unit sphere in $d$-dimensional space will be our constraint set, that is, $\cM =\sphere^{d-1}$. We provide some further backgrounds on how to compute the manifold gradients and Hessian here. 
We view $f$ as the restriction of a smooth function $\bar{f}$ to the manifold $\cM$. In this case, we have  $\tangent_x \cM= \{z\in \R^d: z^{\top} x = 0\}$,  and $P_x = \Id - xx^{\top}$.  We derive the manifold gradient of $f$ on $\cM$:  
$
\grad f(x) = P_x \nabla\bar{f}(x)\mcom\label{eqn:mani-gradient}
$
where $\nabla$ is the usual gradient in the ambient space $\R^d$. Moreover, we derive the Riemannian Hessian as
$
\hessian f(x)   = P_x \nabla^2 \bar{f}(x) P_x - (x^\top \nabla \bar{f}(x)) P_x\label{eqn:mani-hessian}. 
$

%% file: isotonic.tex
\section{Generalized Linear Models}\label{sec:glm}
\newcommand{\hatL}{\widehat{L}}

\newcommand{\<}{\langle}
\renewcommand{\>}{\rangle}
We consider the problem of learning a \emph{generalized linear model} and we will show that the loss function for it will be non-convex, but all of its local minima are global. Suppose we observe $n$ data points $\{(x_i,y_i)\}_{i=1}^n$, where $x_i$'s are sampled i.i.d. from some distribution $D_x$ over $\R^d$. In the generalized linear model, we assume the label $y_i\in \R$ is generated from 
\begin{equation*}
y_i = \sigma(w_\star^\top x_i) + \epsilon_i, \label{eqn:generative}
\end{equation*}
where $\sigma:\R\to\R$ is a known monotone activation function, $\epsilon_i\in\R$ are i.i.d. mean-zero noise (independent with $x_i$), and $w_\star\in\R^d$ is a fixed unknown ground truth coefficient vector. We denote the joint distribution of $(x_i,y_i)$ by $D$. 

Our goal is to recover approximately $w_{\star}$ from the data. We minimize the empirical squared risk:
$
\hatL(w) = \frac{1}{2n}\sum_{i=1}^n (y_i - \sigma(w^\top x_i))^2\mper
$
Let $L(w)$ be the corresponding population risk: 
$L(w) = \frac{1}{2}\Exp_{(x,y)\sim D}\left[(y - \sigma(w^\top x))^2\right].$

We will analyze the optimization of $\widehat{L}$ via characterizing the property of its landscape. Our road map consists of two parts:  a) all the local minima of the population risk are global minima; b) the empirical risk $\widehat{L}$ has the same property.   

When $\sigma$ is the identity function, that is, $\sigma(t) =t$, we have the linear regression problem and the loss function is convex. In practice, people have taken $\sigma$, e.g., to be the sigmoid function and then the objective $\hatL$ is no longer convex. 

Throughout the rest of the section, we make the following regularity assumptions on the problem. These assumptions are stronger than what's necessary, for the ease of exposition. However, we note that some assumptions on the data are necessary because in the worst-case, the problem is intractable. (E.g., the generative assumption~\eqref{eqn:generative} on $y_i$'s is a key one.)
\begin{assumption}\label{ass:1}
	We assume the distribution $D_x$ and activation $\sigma$ satisfy that
	\begin{itemize}
	\item[1.] The vectors $x_i$ are bounded and non-degenerate: $D_x$ is supported in $\{x:\|x\|_2 \le B\}$, and $\E_{x\sim D_x}[xx^\top] \succeq \lambda \Id$ for some $\lambda>0$, where $\Id$ is the identity.
	\item[2.] The ground truth coefficient vector satisfies $\|w_\star\|_2 \le R$, and $BR\ge 1$.
	\item[3.] The activation function $\sigma$ is strictly increasing and twice differentiable. Furthermore, it satisfies the bounds
	\begin{equation*}
	\sigma(t)\in[0,1],~~~\sup_{t\in\R}\left\{|\sigma'(t)|, |\sigma''(t)| \right\}\le 1,~~~{\rm and}~~~\inf_{t\in[-BR, BR]}\sigma'(t) \ge \gamma > 0.
	\end{equation*}
	\item[4.] The noise $\epsilon_i$'s are mean zero and bounded: with probability 1, we have $|\epsilon_i|\le 1$.
\end{itemize}
\end{assumption}

\subsection{Analysis of the population risk}

\newcommand{\newpart}[1]{{#1}}
\newcommand{\brief}[1]{{#1}}
\newcommand{\solution}[1]{{\color{red}#1}}

In this section, we show that all the local minima of the population risk $L(w)$ are global minima. In fact, $L(w)$ has a unique local minimum which is also global.  (But still, $L(w)$ may likely be not convex for many choices of  $\sigma$.)

\begin{theorem}The objective $L(\cdot)$ has a unique local minimum, which is equal to $w_{\star}$ and is also a global minimum. In particular, $L(\cdot)$ is weakly-quasi-convex.  \end{theorem}

The proof follows from directly checking the definition of the quasi-convexity. The intuition is that generalized linear models behave very similarly to linear models from the lens of quasi-convexity: many steps of the inequalities of the proof involves replacing $\sigma$ be an identity function effectively (or replacing $\sigma'$ be 1.)
\begin{proof}[Proof Sketch]
Using the property that $\E[y|x]=\sigma(w_\star^\top x)$, we have the following bias-variance decomposition (which can be derived by elementary manipulation)
	\begin{equation}
	L(w) = \frac{1}{2}\E[(y - \sigma(w^\top x))^2] = \frac{1}{2}\E[(y - \sigma(w_\star^\top x))^2] + \frac{1}{2}\E[(\sigma(w_\star^\top x) - \sigma(w^\top x))^2]\mper\label{eqn:6}
	\end{equation}
	The first term is independent of $w$, and the second term is non-negative and equals zero at $w=w_\star$. Therefore, we see that $w_\star$ is a global minimum of $L(w)$.
	
	Towards proving that $L(\cdot)$ is quasi-convex, we first compute  $\nabla L(w)$: 	\begin{equation*}
	\nabla L(w) =  \E[(\sigma(w^\top x) - y)\sigma'(w^\top x)x] = \E[(\sigma(w^\top x) - \sigma(w_\star^\top x))\sigma'(w^\top x)x],
	\end{equation*}
	where the last equality used the fact that $\E[y|x]=\sigma(w_\star^\top x)$. It follows that	\begin{equation*}
	\<\nabla L(w), w - w_\star\> = \E[(\sigma(w^\top x) - \sigma(w_\star^\top x))\sigma'(w^\top x)\<w-w_\star, x\>].
	\end{equation*}
	Now, by the mean value theorem, 
	and bullet 3 of Assumption~\ref{ass:1}, we have that 		\begin{align*}
	(\sigma(w^\top x) - \sigma(w_\star^\top x))\<w-w_\star, x\>  \ge \gamma(w^\top x - w_\star^\top x)^2.
	\end{align*}
Using $|\sigma'(t)| \ge \gamma$ and $|\sigma'(t)| \le 1$ for every $|t|\le BR$, and the monotonicity of $\sigma$, 
	\begin{align}
	 \quad \<\nabla L(w), w - w_\star\> & = \E[(\sigma(w^\top x) - \sigma(w_\star^\top x))\sigma'(w^\top x)\<w-w_\star, x\>] \nonumber\\
	& \ge \gamma\E(\sigma(w^\top x) - \sigma(w_\star^\top x))(w^\top x - w_\star^\top x)] \label{eqn:5}\\
	& \ge \gamma \E[(\sigma(w^\top x) - \sigma(w_\star^\top x))^2] \nonumber\ge 2\gamma (L(w) - L(w_\star))\label{eqn:7}
	\end{align}
	where the last step uses the decomposition~\eqref{eqn:6} of the risk $L(w)$. 
\end{proof}

\subsection{Concentration of the empirical risk}

We next analyze the  empirical risk $\widehat{L}(w)$. 
We will show that with sufficiently many examples, the empirical risk $\hatL$ is close enough to the population risk $L$ so that $\hatL$ also satisfies that all local minima are global. 
\begin{theorem}[The empirical risk has no bad local minimum]
	\label{thm}
	Under the problem assumptions, 		with probability at least	$1-\delta$, 
for all $w$ with $\|w\|_2\le R$, 
	the
	empirical risk has no local minima outside a small 
	neighborhood of $w_\star$: for any $w$ such that $\|w\|_2\le R$,
	if $\nabla\widehat{L}(w)=0$, then
	\begin{equation*}
	\|w-w_\star\|_2 \le \frac{C_1B}{\gamma^2\lambda}
	\sqrt{\frac{d(C_2 + \log(nBR)) + \log\frac{1}{\delta}}{n}}.
	\end{equation*}
	where $C_1, C_2>0$ are universal constants that do not depend on
	$(B,R,d,n,\delta)$.
\end{theorem}
Theorem~\ref{thm} shows that all stationary points of
$\widehat{L}(w)$ have to be within a small neighborhood of $w_\star$. Stronger landscape property can also be proved though: there is a unique local minimum in the neighborhood of $w_\star$. 
	
The main intuition is that to verify quasi-convexity or restricted secant inequality for $\hatL$,  it suffices to show that with high probability over the randomness of the data, $\forall w \textup{ with } \|w\|_2\le R$
\begin{align}
\<\nabla L(w), w - w_\star\> \approx \<\nabla \hatL(w), w - w_\star\>\mper\label{eqn:9}
\end{align}

Various tools to prove such concentration inequalities have been developed in statistical learning theory  and probability theory community, and a thorough exposition of them is beyond the scope of this chapter.

%% file: pca.tex
\section{Matrix Factorization Problems}\label{sec:matrix}

In this section, we will discuss the optimization landscape of two problems based on matrix factorization: principal component analysis (PCA) and matrix completion. The fundamental difference between them and the generalized linear models is that their objective functions have saddle points that are \textit{not} local minima or global minima. It means that the quasi-convexity condition or Polyak-Lojasiewicz condition does not hold for these objectives. Thus, we need more sophisticated techniques that can distinguish saddle points from local minima.  

\subsection{Principal Component Analysis}
\label{sec:pca}
One interpretation of PCA is approximating a matrix by its best low-rank approximation. Given a matrix $M\in \R^{d_1\times d_2}$, we aim to find its best rank-$r$ approximation (in either Frobenius norm or spectral norm).  For the ease of exposition, we take $r=1$ and assume $M$ to be symmetric positive semi-definite with dimension $d$ by $d$. In this case, the best rank-1 approximation has the form $xx^\top$ where $x\in \R^d$. 

There are many well-known algorithms for finding the low-rank factor $x$. We are particularly interested in the following non-convex program that directly minimizes the approximation error in Frobenius norm.  
\begin{align}
\min_x ~g(x) := \frac12\cdot \norm{ M-xx^{\top}}_F ^2 \mper \label{eqn:g}
\end{align}
We will prove that even though $g$ is not convex, all the local minima of $g$ are global. It also satisfies the strict saddle property (which we will not prove formally here). Therefore, local search algorithms can solve~\eqref{eqn:g} in polynomial time.\footnote{In fact, local methods can solve it very fast. See, e.g., \citet[Thereom 1.2]{li2017algorithmic}}

\begin{theorem}\label{thm:pca}
In the setting above, all the local minima of the objective function $g(x) $ are global minima.\footnote{The function $g$ also satisfies the $(\alpha,\beta,\gamma)$-strict-saddle property (Definition~\ref{def:strictsaddle}) with some $\alpha, \beta, \gamma > 0$ (that may depend on $M$) so that it satisfies the condition of Theorem~\ref{thm:global}. We skip the proof of this result for simplicity.}
\end{theorem}

Our analysis consists of two main steps: a) to characterize all the stationary points of the function $g$, which turn out to be the eigenvectors of $M$; b) to examine each of the stationary points and show that the only the top eigenvector(s) of $g$ can be a local minimum. Step b) implies the theorem because  the top eigenvectors are also global minima of $g$. We start with step a) with the following lemma.

\begin{lemma}\label{lem:pcastaionary}
	In the setting of Theorem~\ref{thm:pca}, all the stationary points of the objective $g()$ are the eigenvectors of $M$. Moreover, if $x$ is a stationary point, then $\|x\|_2^2$ is the eigenvalue corresponding to $x$. 
\end{lemma}

\begin{proof}
	By elementary calculus, we have that 
	\begin{align}
	\nabla g(x) = -(M-xx^\top)x = \|x\|_2^2\cdot x - Mx
	\end{align}
	Therefore, if $x$ is a stationary point of $g$, then $Mx = \|x\|_2^2 \cdot x$, which implies that $x$ is an eigenvector of $M$ with eigenvalue equal to $\|x\|_2^2$. 
\end{proof}

Now we are ready to prove b) and the theorem. The key intuition is the following. Suppose we are at a point $x$ that is an eigenvector but not the top eigenvector, moving in either the top eigenvector direction $v_1$ or the direction of $-v_1$ will result in a second-order local improvement of the objective function. Therefore, $x$ cannot be a local minimum unless $x$ is a top eigenvector.
\begin{proof}[Proof of Theorem~\ref{thm:pca}]
	By Lemma~\ref{lem:pcastaionary}, we know that a local minimum $x$ is an eigenvector of $M$. If $x$ is a top eigenvector of $M$ with the largest eigenvalue, then $x$ is a global minimum. For the sake of contradiction, we assume that $x$ is an eigenvector with eigenvalue $\lambda$ that is strictly less than $\lambda_1$. By Lemma~\ref{lem:pcastaionary} we have $\lambda = \|x\|_2^2$. 
By elementary calculation, we have that  
\begin{align}
\nabla^2 g(x) =  2xx^{\top} -  M + \|x\|_2^2 \cdot I\mper\label{eqn:2}
\end{align}

Let $v_1$ be the top eigenvector of $M$ with eigenvalue $\lambda_1$ and with $\ell_2$ norm 1. 
Then, because $\nabla^2 g(x) \succeq 0$, we have that 
\begin{align}
v_1^\top \nabla^2 g(x) v \ge 0  \label{eqn:1}
\end{align}
It's a basic property of eigenvectors of positive semidefinite matrix that any pairs of eigenvectors with different eigenvalues are orthogonal to each other. Thus we have $\inner{x, v_1} = 0$. 
It follows equation~\eqref{eqn:1} and~\eqref{eqn:2} that 
\begin{align}
0 & \leq  v_1^\top (2xx^{\top} -  M + \|x\|_2^2 \cdot I) v_1 =  \|x\|_2^2 - v_1^\top M v_1 \tag{by $\inner{x, v_1} = 0$}\\
&=  \lambda - \lambda_1 \tag{by that $v_1$ has eigenvalue $\lambda_1$ and that $\lambda=\|x\|_2^2$}\\
& < 0 \tag{by the assumption}
\end{align}
which is a contradiction. \end{proof}

%% file: mc.tex
\subsection{Matrix completion}
\label{sec:mc}

Matrix completion is the problem of recovering a low-rank matrix from partially observed entries, which has been widely used in collaborative filtering and recommender systems, dimension reduction, and multi-class learning. Despite the existence of elegant convex relaxation solutions, stochastic gradient descent on non-convex objectives are widely adopted in practice for scalability. We will focus on the rank-1 symmetric matrix completion in this chapter, which demonstrates the essence of the analysis.

\subsubsection{Rank-1 case of matrix completion}

Let $M = zz^{\top}$ be a rank-1 symmetric matrix with factor $z \in \R^{d}$ that we aim to recover.  We assume that we observe each entry of $M$ with probability $p$ independently.\footnote{Technically, because $M$ is symmetric, the entries at $(i,j)$ and $(j,i)$ are the same. Thus, we assume that, with probability $p$ we observe both entries and otherwise we observe neither.}
Let $\Omega \subset [d]\times [d]$ be the set of entries observed. 

Our goal is to recover from the observed entries of $M$ the vector $z$ up to sign flip (which is equivalent to recovering $M$). 

A known issue with matrix completion is that if $M$ is ``aligned'' with standard basis, then it's impossible to recover it. E.g., when $M = e_je_j^\top$ where $e_j$ is the $j$-th standard basis, we will very likely observe only entries with value zero, because $M$ is sparse. Such scenarios do not happen in practice very often though. The following standard assumption will rule out these difficult and pathological cases: 

\begin{assumption}[Incoherence] \label{assump:incoherence}
W.L.O.G, we assume that $\|z\|_2=1$. In addition, we assume that $z$ satisfies 
	$
	\|z\|_\infty \le \frac{\mu}{\sqrt{d}}.
	$  We will think of $\mu$ as a small constant or logarithmic in $d$, and the sample complexity will depend polynomially on it.
\end{assumption}

In this setting, the vector $z$ can be recovered exactly up to a sign flip provided $\widetilde{\Omega}(d)$ samples. However, for simplicity, in this subsection we only aim to recover $z$ with an $\ell_2$ norm error $\eps \ll 1$.  We assume that $p = \mbox{poly}(\mu,\log d)/(d\epsilon^2)$ which means that the expected number of observations is on the order of $d/\eps\cdot \polylog d$.  
We analyze the following objective  that minimizes the total squared errors on the observed entries:
\begin{align}
\argmin_x ~f(x) := \frac12 \sum_{(i,j)\in \Omega}(M_{ij}-x_ix_j)^2= \frac12\cdot \norm{\Po( M-xx^{\top})}_F ^2 \mper
\end{align}
Here $P_\Omega(A)$ denotes the matrix obtained by zeroing out all the entries of $A$ that are not in $\Omega$. 
For simplicity, we only focus on characterizing the landscape of the objective in the following domain $\mathcal{B}$ of incoherent vectors that contain the ground-truth vector $z$ (with a buffer of factor of 2) \begin{align}
\mathcal{B} & = \Set{x: \|x\|_{\infty}  < \frac{2\mu}{\sqrt{d}}} \mper\label{eqn:x_incoherent_0} 
\end{align}

We note that the analyzing the landscape inside $\mathcal{B}$ does not suffice because the iterates of the algorithms may leave the set $\mathcal{B}$. We refer the readers to the original paper~\citep{ge2016matrix} for an analysis of the landscape over the entire space, or to the recent work~\citep{ma2018implicit} for an analysis that shows that the iterates won't leave the set of incoherent vectors if the initialization is random and incoherent.

The global minima of $f(\cdot)$ are $z$ and $-z$ with function value 0. In the rest of the section, we prove that all the local minima of $f(\cdot)$ are $O(\sqrt{\epsilon})$-close to $\pm z$.

\begin{theorem}\label{lem:mc:partial}
In the setting above,  all the local minima of $f(\cdot)$ inside the set $\mathcal{B}$ are $O(\sqrt{\epsilon})$-close to either $z$ or $-z$.\footnote{It's also true that the only local minima are exactly $\pm z$, and that $f$ has strict saddle property. However, their proofs are involved and beyond the scope of this chapter.}
\end{theorem}

It's insightful to compare with the full observation case when $\Omega = [d]\times [d]$. The corresponding objective is exactly the PCA objective $g(x) = \frac12\cdot \norm{ M-xx^{\top}}_F ^2$ defined in equation~\eqref{eqn:g}. Observe that $f(x)$ is a sampled version of the  $g(x)$, and therefore we expect that they share the same geometric properties. In particular, recall that $g(x)$ does not have spurious local minima and thus we expect neither does $f(x)$. 

However, it‘s non-trivial to extend the proof of Theorem~\ref{thm:pca} to the case of partial observation, because it uses the {\em properties of eigenvectors} heavily. Indeed, suppose we  imitate the proof of Theorem~\ref{thm:pca}, we will first compute the gradient of $f(\cdot)$: 
\begin{align}
& \nabla f(x) = P_{\Omega}(zz^{\top}-xx^{\top})x\mper
\end{align}
Then, we run into an immediate difficulty --- how shall we solve the equation for stationary points $f(x) = P_{\Omega}(M-xx^{\top})x= 0$. Moreover, even if we could have a reasonable approximation for the stationary points, it would be difficult to examine their Hessians  without using the exact orthogonality of the eigenvectors. 

The lesson from the trial above is that we may need to have an alternative proof for the PCA objective (full observation) that relies less on solving the stationary points exactly. Then more likely the proof can be extended to the matrix completion (partial observation) case. In the sequel, we follow this plan by first providing an alternative proof for Theorem~\ref{thm:pca}, which does not require solving the equation $\nabla g(x) = 0$, and then extend it via concentration inequality to a proof of Theorem~\ref{lem:mc:partial}.   The key intuition will be is the following:

\vspace{.05in}
\textit{Proofs that consist of inequalities that are linear in $\mathbf{1}_{\Omega}$ are often easily generalizable to partial observation case.} 
\vspace{.05in}

\noindent 
\sloppy Here statements that are linear in $\mathbf{1}_{\Omega}$ mean the statements of the form $\sum_{ij} 1_{(i,j) \in \Omega} T_{ij} \le a$. We will call these kinds of proofs ``simple'' proofs in this section. 
Indeed, by the law of large numbers, when the sampling probability $p$ is sufficiently large, we have that 
\begin{align}
\underbrace{\sum_{(i,j)\in \Omega} T_{ij}}_{\textrm{partial observation}} = \sum_{i,j} \mathbf{1}_{(i,j)\in \Omega}T_{ij} \approx p \underbrace{\sum_{i,j}T_{ij}}_{\textrm{full observation}}\label{eqn:partialfull}
\end{align}
Then, the mathematical implications of $p\sum T_{ij} \le a$ are expected to be similar to the implications of $\sum_{(i,j)\in \Omega} T_{ij} \le a/p$, up to some small error introduced by the approximation.

What natural quantities about $f$ are of the form $\sum_{(i,j)\in \Omega} T_{ij}$? First, quantities of the form $\inner{P_\Omega(A), B}$ can be written as $\sum_{(i,j)\in \Omega} A_{ij} B_{ij}$. Moreover, both the projection of $\nabla f$ and $\nabla^2 f$ are of the form $\inner{P_\Omega(A), B}$:
\begin{align}
\langle v, \nabla f(x)\rangle & = \inner{v,P_{\Omega}(zz^{\top}-xx^{\top})x} = \inner{P_{\Omega}(zz^{\top}-xx^{\top}), vx^\top} \nonumber\\
\langle v, \nabla^2 f(x) v\rangle & = \inner{P_{\Omega}(vx^{\top}+xv^{\top}), vx^{\top}+xv^{\top}}_F^2 - 2\inner{P_{\Omega}(vv^{\top}-xx^{\top}), vv^\top} \nonumber
\end{align}
The concentration of these quantities can all be captured by the following theorem below: 
\begin{theorem}\label{thm:concentration}
		Let $\eps > 0$ and $ p = \mbox{poly}(\mu,\log d)/(d\epsilon^2)$. Then, with high probability of the randomness of $\Omega$, we have that for all $A = uu^\top , B = vv^\top \in \R^{d\times d}$, where $\|u\|_2 \le 1, \|v\|2\le 1$ and $\|u\|_\infty, \|v\|_\infty \leq 2\mu/\sqrt{d}$. 
\begin{align}
|\inner{P_\Omega(A), B}/p - \inner{A,B}| \le \eps\mper\end{align}
\end{theorem}

We will provide two claims below, combination of which proves Theorem~\ref{thm:pca}. In the proofs of these two claims, all the inequalities are of the form of LHS of equation~\eqref{eqn:partialfull}. Following each claim, we will immediately provide its extension to the partial observation case.

\begin{customthm}{1f}\label{claim:full1}
	Suppose $x\in \mathcal{B}$ satisfies $\nabla g(x) = 0$, then $\inner{x,z}^2 = \|x\|_2^4$.		
\end{customthm}
\begin{proof} By elementary calculation
	\begin{align}
	& \nabla g(x) = (zz^{\top}-xx^{\top})x= 0  \nonumber\\
	\Rightarrow ~~& \inner{x,\nabla g(x)} = \inner{x,(zz^{\top}-xx^{\top})x} = 0 \quad\quad \quad \label{eqn:91}\\
	\Rightarrow ~~& \inner{x,z}^2 =\|x\|_2^4 \nonumber
	\end{align}
	Intuitively, a stationary point $x$'s norm is governed by its correlation with $z$. 
\end{proof}
\ni The following claim is the counterpart of Claim~\ref{claim:full1} in the partial observation case. 
\begin{customthm}{1p}\label{claim:partial1}
	Suppose $x\in \mathcal{B}$ satisfies $\nabla f(x) = 0$, then $\inner{x,z}^2 \ge \|x\|^4 -\epsilon$.		
\end{customthm}
\begin{proof} Imitating the proof of Claim~\ref{claim:full1}, 
	\begin{align}
	& \nabla f(x) = P_{\Omega}(zz^{\top}-xx^{\top})x= 0  \nonumber\\
	\Rightarrow ~~& \inner{x,\nabla f(x)} = \inner{x,P_{\Omega}(zz^{\top}-xx^{\top})x} = 0\quad\quad \label{eqn:92}\\
		\Rightarrow ~~&  \inner{x,\nabla g(x)} = |\inner{x,(zz^{\top}-xx^{\top})x}| \leq \eps \quad\quad \label{eqn:13} \\
	\Rightarrow ~~& \inner{x,z}^2 \ge \|x\|_2^4 -\epsilon\nonumber\end{align}
	where derivation from line~\eqref{eqn:92} to ~\eqref{eqn:13} follows the fact that line~\eqref{eqn:92} is a sampled version of~\eqref{eqn:13}. Technically, we can obtain it by applying Theorem~\ref{thm} twice with $A = B = xx^\top$ and $A=xx^\top$ and $B=zz^\top$ respectively. \end{proof}

\begin{customthm}{2f}\label{claim:full2}
	If $x\in \mathcal{B}$ has positive Hessian $\nabla^2 g(x)\succeq 0$, then $\|x\|_2^2 \ge 1/3$.
\end{customthm}

\begin{proof}
	By the assumption on $x$, we have that $\inner{z,\nabla^2 g(x)z} \ge 0$. Calculating the quadratic form of the Hessian (which can be done by elementary calculus and is skipped for simplicity), we have 
	\begin{align}
	& \inner{z,\nabla^2 g(x)z} = \norm{zx^{\top}+xz^{\top}}_F^2 - 2z^{\top}(zz^{\top}-xx^{\top})z\ge 0 \label{eqn:14}
	\end{align}
	This implies that 
	\begin{align}
	& \Rightarrow \norm{x}_2^2 + 2\inner{z,x}^2 \ge 1 \nonumber\\
	& \Rightarrow \norm{x}_2^2\ge 1/3\tag{since $\inner{z,x}^2\le \norm{x}_2^2$}
	\end{align} 
\end{proof}

\begin{customthm}{2p}\label{claim:partial2}
	If $x\in \mathcal{B}$ has positive Hessian $\nabla^2 f(x)\succeq 0$, then $\|x\|_2^2 \ge 1/3-\epsilon/3$.
\end{customthm}
\begin{proof}
	Imitating the proof of Claim~\ref{claim:full2}, calculating the quadratic form over the Hessian at $z$, we  have 
	\begin{align}
 \inner{z,\nabla^2 f(x)z}  = \norm{P_{\Omega}(zx^{\top}+xz^{\top})}_F^2 - 2z^{\top}P_{\Omega}(zz^{\top}-xx^{\top})z\ge 0 
	\end{align}
Note that equation above is just a sampled version of equation~\eqref{eqn:14}, applying Theorem~\ref{thm:concentration} for various times (and note that $\inner{P_{Omega}(A), P_{\Omega}(B)} = \inner{P_\Omega(A), B}$, we can obtain that 
\begin{align}
& \norm{P_{\Omega}(zx^{\top}+xz^{\top})}_F^2 - 2z^{\top}P_{\Omega}(zz^{\top}-xx^{\top})z \nonumber\\
& = p \cdot \left(\norm{zx^{\top}+xz^{\top}}_F^2 - 2z^{\top}(zz^{\top}-xx^{\top})z \pm \eps \right)\nonumber
\end{align}
Then following the derivation in the proof of Claim~\ref{claim:full2}, we achieve the same conclusion of Claim~\ref{claim:full2} up to approximation:
$
\norm{x}^2\ge 1/3-\epsilon/3\nonumber. 
$
\end{proof}

\ni With these claims, we are ready to prove Theorem~\ref{thm:pca} (again) and Theorem~\ref{lem:mc:partial}.

\begin{proof}[Proof of Theorem~\ref{thm:pca} (again) and Theorem~\ref{lem:mc:partial}]
	By Claim~\ref{claim:full1} and~\ref{claim:full2}, we have $x$ satisfies $\inner{x,z}^2\ge \|x\|_2^4\ge 1/9$.
	Moreover, we have that $\nabla g(x) =  0 $ implies 
	\begin{align}
	& \inner{z,\nabla g(x)} = \inner{z,(zz^{\top}-xx^{\top})x} = 0\label{eqn:93}\\
	\Rightarrow ~~& \inner{x,z} (1 - \|x\|_2^2 ) = 0\nonumber\\
	\Rightarrow ~~& \|x\|_2^2  = 1 \tag{by $\inner{x,z}^2\ge 1/9$} 
	\end{align}
	Then by Claim~\ref{claim:full1} again we obtain $\inner{x,z}^2 = 1$, and therefore $x = \pm z$. 

	The proof of Theorem~\ref{lem:mc:partial} are analogous (and note that such analogy was by design). When $\eps\le 1/12$, we by Claim~\ref{claim:partial1} and~\ref{claim:partial2}, we have when $\eps\le 1/16$, 
	\begin{align}
\inner{x,z}^2\ge \|x\|_2^4 - \eps\ge \left(\frac{1-\eps}{3}\right)^2 -\eps \ge \frac{1}{32} 
	\end{align}
Because 	$|\inner{z,\nabla g(x)} - \inner{z,\nabla f(x)}/p|\le \eps$, we have that 
	\begin{align}
	& |\inner{z,\nabla g(x)}| = |\inner{z,(zz^{\top}-xx^{\top})x}| = O(\eps)\nonumber\\
	\Rightarrow ~~& |\inner{x,z} (1 - \|x\|_2^2 ) |=  O(\eps)\nonumber\\
	\Rightarrow ~~& \|x\|_2^2  = 1 \pm O(\eps) \tag{by $\inner{x,z}^2\ge 1/32$} 
	\end{align}
	Then by Claim~\ref{claim:partial1} again, we have $\inner{x,z}^2 \ge 1-O(\eps)$ which implies that $|\inner{x,z}| \ge 1-O(\eps)$.  Now suppose $\inner{x,z} \ge 1-O(\eps)$, then we have 
	\begin{align}
	\|x-z\|_2^2 = \|x\|_2^2 + \|z\|_2^2 - 2\inner{x,z} \le 1+O(\eps) + 1 - (1-O(\eps) \le O(\eps)\nonumber
	\end{align}
	Therefore $x$ is $O(\sqrt{\eps})$ close to $z$. On the other hand, if $\inner{x,z} \le -(1-O(\eps))$, we can similarly conclude that $x$ is $O(\sqrt{\eps})$-close to $-z$. 
\end{proof}

%% file: tensor.tex
\section{Landscape of Tensor Decomposition}
\label{sec:tensor}
In this section, we analyze the optimization landscape for another machine learning problem, tensor decomposition. The fundamental difference of tensor decomposition from matrix factorization  problems or generalized linear models is that the non-convex objective function here has multiple isolated local minima, and therefore the set of local minima does not have rotational invariance (whereas in matrix completion or PCA, the set of local minima are rotational invariant.). This essentially prevents us to only use linear algebraic techniques, because they are intrinsically rotational invariant.

\subsection{Non-convex optimization for orthogonal tensor decomposition and global optimality}

We focus on one of the simplest tensor decomposition problems, orthogonal 4-th order tensor decomposition.  Suppose we are given the entries of a symmetric 4-th order tensor $T \in \R^{d \times d \times d\times d}$ which has a low rank structure in the sense that: 
\begin{align}
T = \sum_{i=1}^n a_i\otimes a_i\otimes a_i \otimes a_i \label{eqn:lowrank}
\end{align}
where $a_1,\dots, a_n \in \R^d$. Our goal is to recover the underlying components $a_1,\dots, a_n$.  We assume in this subsection that $a_1,\dots, a_n$ are orthogonal vectors in $\R^d$ with unit norm (and thus implicitly we assume $n\le d$.) Consider the objective function
\begin{align}
\argmax & ~~f(x) := \inner{T, x^{\otimes 4}}\label{eqn:obj}\\
\textup{s.t.} &~~ \|x\|_2^2 = 1 \nonumber
\end{align}
The optimal value function for the objective is the (symmetric) injective norm of a tensor $T$. In our case, the global maximizers of the objective above are exactly the set of components that we are looking for. 
\begin{theorem}\label{thm:tensorglobal}
	Suppose $T$ satisfies equation~\eqref{eqn:lowrank} with orthonormal components $a_1,\dots, a_n$. Then, the global maximizers of the objective function~\eqref{eqn:obj} are exactly $\pm a_1, \dots, \pm a_n$. 
\end{theorem}

\subsection{All local optima are global}
We next show that all the local maxima of the objective~\eqref{eqn:obj} are also global maxima. In other words, we will show that $\pm a_1,\dots, \pm a_n$ are the only local maxima. We note that all the geometry properties here are defined with respect to the manifold of the unit sphere $\cM = \sphere^{d-1}$. (Please see Section~\ref{sec:manifold} for a brief introduction of the notions of manifold gradient, manifold local maxima, etc.)

\begin{theorem}\label{thm:tensorlocal}
In the same setting of Theorem~\ref{thm:tensorglobal}, all the local maxima (w.r.t the manifold $\sphere^{d-1}$) of the objective~\eqref{eqn:obj} are global maxima. \footnote{The function also satisfies the strict saddle property so that we can rigorously invoke Theorem~\ref{thm:manifold}. However, we skip the proof of that for simplicity.}
\end{theorem}

Towards proving the Theorem, we first note that the landscape property of a function is invariant to the coordinate system that we use to represent it. It's natural for us to use the directions of $a_1,\dots, a_n$ together with an arbitrary basis in the complement subspace of $a_1,\dots, a_n$ as the coordinate system. A more convenient viewpoint is that this choice of coordinate system is equivalent to assuming $a_1,\dots, a_n$ are the natural standard basis $e_1,\dots, e_n$. Moreover, one can verify that the remaining directions $e_{n+1},\dots, e_d$ are irrelevant for the objective because it's not economical to put any mass in those directions. Therefore, for simplicity of the proof, we make the assumption below without loss of generality: 
\begin{align}
n= d, \textup{ and } a_i = e_i, ~\forall i\in [n]\mper
\end{align} 
Then we have that $f(x) = \|x\|_4^4$.
We compute the manifold gradient and manifold Hessian using the formulae of $\grad f(x)$ and $\hessian f(x)$ in Section~\ref{sec:manifold}, 
\begin{align}
\grad f(x) &= 4P_x \nabla \bar{f}(x) = 4(\Id_{d\times d}-xx^\top) \begin{bmatrix}
x_1^3 \\
\vdots \\
x_d^3
\end{bmatrix}= 4\begin{bmatrix}
x_1^3  \\
\vdots \\
x_d^3
\end{bmatrix} - 4\|x\|_4^4 \cdot \begin{bmatrix}
x_1 \\
\vdots \\
x_d
\end{bmatrix}\mper\label{eqn:grad}
\end{align}
\begin{align}
\hessian f(x) & = P_x \nabla^2 \bar{f}(x) P_x - (x^\top \nabla \bar{f}(x)) P_x \nonumber\\
& = P_x\left(12\diag(x_1^2,\dots, x_d^2) - 4\|x\|_4^4 \cdot \Id_{d\times d}\right) P_x \label{eqn:tensorhessian}
\end{align}
where $\diag(v)$ for a vector $v\in \R^d$ denotes the diagonal matrix with $v_1,\dots, v_d$ on the diagonal. 
Now we are ready to prove Theorem~\ref{thm:tensorlocal}. In the proof, we will first compute all the stationary points of the objective, and then examine each of them and show that only $\pm a_1,\dots, \pm a_n$ can be local maxima.
\begin{proof}[Proof of Theorem~\ref{thm:tensorlocal}]
We work under the assumptions and simplifications above. We first compute all the stationary points of the objective~\eqref{eqn:obj} by solving $\grad f = 0$. Using equation~\eqref{eqn:grad}, we have that the stationary points satisfy that 
	\begin{align}
	x_i^3 = \|x\|_4^4 \cdot x_i, \forall i
	\end{align}
	It follows that $x_i = 0$ or $x_i = \pm \|x\|_4^{1/2}$. Assume that $s$ of the $x_i$'s are non-zero and thus take the second choice, we have that 
	\begin{align}
	1 = \|x\|_2^2 = s\cdot \|x\|_4^4
	\end{align}
	This implies that $\|x\|_4^4 = 1/s$, and $x_i = 0$ or $\pm 1/s^{1/2}$.  In other words, all the stationary points of $f$ are of the form $(\pm 1/s^{1/2}, \cdots, \pm 1/s^{1/2}, 0,\cdots, 0)$  (where there are $s$ non-zeros) for some $s\in [d]$ and all their permutations (over indices).  
	
	Next, we examine which of these stationary points are local maxima. Let $\tau = 1/s^{1/2}$ for simplicity. This implies that $\|x\|_4^4 = \tau^2$. Consider a stationary point $x = (\sigma_1\tau, \cdots, \sigma_s\tau, 0,\dots, 0)$ where $\sigma_i \in \{-1,1\}$. Let $x$ be a local maximum. Thus $\hessian f(X) \preceq 0$. We will prove that this implies $s= 1$. For the sake of contradiction, we assume $s\ge 2$. We will show that the Hessian cannot be negative semi-definite by finding a particular direction in which the Hessian has positive quadratic form. 
	
	The form of equation\eqref{eqn:tensorhessian} implies that for all $v$ such that $\inner{v,x} = 0$ (which indicates that $P_x v = v$), we have 
	\begin{align}
	v^\top \left((12\diag(x_1^2,\dots, x_d^2) - 4\|x\|_4^4 \Id\right) v \leq 0 \label{eqn:3}
	\end{align}
	We take $v = (1/2, -1/2)$ to be our test direction. Then LHS of the formula above simplifies to
	\begin{align}		
	3x_1^2 - 3x_2^2 - 2 \|x\|_4^4 = 6\tau^2 - 2 \|x\|_4^4 = 4\tau^2 > 0
	\end{align}
	which contradicts to equation~\eqref{eqn:3}. Therefore, $s = 1$, and we conclude that all the local maxima are $\pm e_1.\dots, \pm e_d$. 
\end{proof}

%% file: outlook.tex
\section{Survey and Outlook: Optimization of Neural Networks}\label{sec:neural_net}

Theoretical analysis of algorithms for learning neural networks is highly challenging. We still lack handy mathematical tools. We will articulate a few technical challenges and summarize the attempts and progresses.

We follow the standard setup in supervised learning. Let $f_{\theta}$ be a neural network parameterized by parameters $\theta$.\footnote{E.g., a two layer neural network would be $f_{\theta}(x) = W_1\sigma(W_2x)$ where $\theta = (W_1,W_2)$ and $\sigma$ are some activation functions.} Let $\ell$ be the loss function, and $\{(x^{(i)},y^{(i)})\}_{i=1}^n$ be a set of i.i.d examples drawn from distribution $D$. The empirical risk is 
$
\hatL(\theta) = \frac{1}{n}\sum_{i=1}^{n}\ell(f_\theta(x^{(i)}), y^{(i)}), 
$
and the population risk is 
$
L(\theta) = \Exp_{(x,y)\sim D}\left[\ell(f_\theta(x), y)\right]. 
$

The major challenge of analyzing the landscape property of $\hatL$ or $L$ stems from the non-linearity of neural networks---$f_{\theta}(x)$ is neither linear in $x$, nor in $\theta$. As a consequence, $\hatL$ and $L$ are not convex in $\theta$. Linear algebra is at odds with neural networks---neural networks do not have good invariance property with respect to rotations of parameters or data points. 

\vspace{0.1in} \noindent{\bf  Linearized neural networks.} Early works for optimization in deep learning simplify the problem by considering linearized neural networks: $f_\theta$ is assumed to be a neural networks without any activations functions. E.g., $f_{\theta} = W_1W_2W_3 x$ with $\theta = (W_1,W_2,W_3)$ would be a three-layer feedforward linearized neural network. Now, the model $f_\theta$ is still not linear in $\theta$, but it is linear in $x$. This simplification maintains the property that $\hatL$ or $L$ are still nonconvex functions in $\theta$, but allows the use of linear algebraic tools to analyze the optimization landscapes of $\hatL$ or $L$. 

\citet{baldi1989neural,kawaguchi2016deep} show that all the local minima of $L(\theta)$ are global minima when $\ell$ is the squared loss and $f_\theta$ is a linearized feed-forward neural network (but $L(\theta)$ does have degenerate saddle points so that it does not satisfy the strict saddle property). ~\citet{hardt2016gradient,hardt17identity} analyze the landscape of learning linearized residual and recurrent neural networks  and show that all the stationary points (in a region) are global minima.  We refer the readers to~\citet{arora2018optimization} and references therein for some recent works along this line. 

There are various results on another simplification: two-layer neural networks with quadratic activations. In this case, the model $f_{\theta}(x)$ is linear in $x\otimes x$ and quadratic in the parameters, and linear algebraic techniques allow us to obtain relatively strong theory. See~\citet{li2017algorithmic, soltanolkotabi2018theoretical,du2018power} and references therein.

We remark that the line of results above typically applies to the landscape of the population losses as well as the empirical losses when there are sufficient number of examples.\footnote{Note that the former implies the latter when there are sufficient number of data points compared to the number of parameters, because in this case, the empirical loss has a similar landscape to that of the population loss due to concentration properties~\citep{mei2017landscape}. }

\vspace{0.1in} \noindent{\bf  Changing the landscape, by, e.g., over-parameterization or residual connection.} 
Somewhat in contrast to the clean case covered in earlier sections of this chapter, people have empirically found that the landscape properties of neural networks depend on various factors including the loss function, the model parameterization, and the data distribution. In particular, changing the model parameterization and the loss functions properly could ease the optimization. 

An effective approach to changing the landscape is to over-parameterize the neural networks --- using a large number of parameters by enlarging the width, often not necessary for expressivity and often bigger than the total number of training samples. It has been empirically found that wider neural networks may alleviate the problem of bad local minima that may occur in training narrower nets~\citep{livni2014computational}.  
This motivates a lot of studies of on the optimization landscape of over-parameterized neural networks. Please see~\citet{safran2016quality,venturi2018neural,soudry2016no,haeffele2015global} and the references therein.

We note that there is an important distinction between two type of overparameterizations: (a) more parameters than what's needed for sufficient expressivity but still fewer parameters than the number of training examples, and (b) more parameters than the number of training examples. Under the latter setting, analyzing the landscape of empirical loss no longer suffices because even if the optimization works, the generalization gap might be too large or in other words the model overfits (which is an issue that is manifested clearly in the NTK discussion below.) In the former setting, though the generalization is less of a concern, analyzing the landscape is more difficult because it has to involve the complexity of the ground-truth function. 

Two extremely empirically successful approaches in deep learning, residual neural networks~\citep{he15deepresidual} and batch normalization~\citep{ioffe2015batch} are both conjectured to be able to change the landscape of the training objectives and lead to easier optimization. This is an interesting and promising direction with the potential of circumventing certain mathematical difficulties, but existing works often suffers from the strong assumptions such as linearized assumption in ~\citet{hardt17identity} and the Gaussian data distribution assumption in ~\citet{ge2018learning}.

\vspace{0.1in} \noindent{\bf  Connection between over-parametrized model and Kernel method: the Neural Tangent Kernel (NTK) view. }
Another recent line of work studies the optimization dynamics of learning over-parameterized neural networks at a special type of initialization with a particular learning rate scheme~\citep{li2018learning,du2018gradient,jacot2018neural,allen2018convergence}, instead of characterizing the full landscape of the objective function. 
The main conclusion is of the following form:when using overparameterization (with more parameters than training examples), under a special type of initialization, optimizing with gradient descent can 
converge to a zero training error solution. 

The results can also be viewed/interpreted as a combination of landscape results and a convergence result: (i) the landscape in a small neighborhood around the initialization is sufficiently close to be convex, (ii) in the neighborhood a zero-error global minimum exists, and (iii) gradient descent from the initialization will not leave the neighborhood and will converge to the zero-error solution. 
Consider a non-linear model $f_{\theta}(\cdot)$ and an initialization $\theta_0$. We can approximate the model by a linear model  by Taylor expansion at $\theta_0$:
\begin{align}
f_\theta(x) \approx g_\theta(x)  & \triangleq \langle \theta-\theta_0, \nabla f_{\theta_0}(x)\rangle + f_{\theta_0}(x)  = \langle \theta, \nabla f_{\theta_0}(x)\rangle +  c(x)\label{eqn:approx}
\end{align}
where $c(x)$ only depends on $x$ but not $\theta$. Ignoring the non-essential shift $c(x)$,  the model $g_{\theta}$ can be viewed as a linear function over the feature vector $\nabla f_{\theta_0}(x)$. 
Suppose the approximation in~\eqref{eqn:approx} is accurate enough throughout the training, then we are essentially optimizing the linear model $g_{\theta}(x)$, which leads to the part (i). \textit{For certain settings of initialization}, it turns out that (ii) and (iii) can also be shown with some proper definition of the neighborhood.

\vspace{0.1in} \noindent{\bf  Limitation of NTK and beyond.} A common limitation of  analyses based on NTK is that they analyze directly the empirical risk whereas they do not necessarily provide good enough generalization guarantees. This is partially because the approach cannot handle regularized neural networks and the particular learning rate and level of stochasticity used in practice.  In practice, typically the parameter $\theta$ does not stay close to the initialization either because of a large initial learning rate or small batch size. When the number of parameters in $\theta$ is bigger than $n$, without any regularization, we cannot expect that $\hatL$ uniformly concentrates around the population risk. This raises the question of whether the obtained solution simply memorizes the training data and does not generalize to the test data. A generalization bound can be obtained by the NTK approach, by bounding the norm of the difference between the final solution and the initialization. However, such a generalization bound can only be effectively as good as what a kernel method can provide. In fact, \citet{wei2019regularization} show that, for a simple distribution, NTK has fundamentally worse sample complexity than a regularized objective for neural networks. This result demonstrates that the NTK regime of neural nets is statistically not as powerful as regularized neural nets, but it does not show that the regularized neural net can be optimized efficiently. 
Many recent works aim to separate training neural net and its NTK regime in a computationally-efficient sense, that is, to present a polynomial time algorithm of training neural networks that enjoys a better generalization guarantee than what the NTK result can offer. E.g., ~\cite{li2020learning} show that gradient descent can learn a two-layer neural net with orthonormal weights on a Gaussian data distribution from small but random initialization with a sample complexity better than the NTK approach.  \citet{allen2019can} present a family of functions that can be learn efficiently by three-layer neural networks but not by NTK. These results are still largely demonstrating the possibility of stronger results for neural networks than NTK on some special cases, and it remains a major open question to have more general analysis for neural networks optimization beyond NTK. 

\vspace{0.03in} \noindent{\bf  Regularized neural networks.}
Analyzing the landscape or optimization of a regularized objective is more challenging than analyzing the un-regularized ones. In the latter case, we  know that achieving zero training loss implies that we reach a global minimum, whereas in the former case, we know little about the function value of the global minima. Some progresses had been made for infinite-width two-layer neural networks~\citep{chizat2018global, mei2018mean,wei2019regularization,sirignano2018mean,rotskoff2018neural}. For example,~\citet{wei2019regularization} show that polynomial number of iterations of perturbed gradient descent can find a global minimum of an $\ell_2$ regularized objective function for infinite-width two-layer neural networks with homogeneous activations. However, likely the same general result won't hold for polynomial-width neural networks, if we make no additional assumptions on the data. 

\vspace{0.03in} \noindent{\bf  Algorithmic or implicit regularization.}
Empirical findings suggest, somewhat surprisingly, that even unregularized neural networks with over-parameterization can generalize~\citep{zhang2016understanding}. Moreover, different algorithms apparently converge to essentially differently global minima of the objective function, and these global minima have \textit{different} generalization performance! This means that the algorithms have a regularization effect, and fundamentally there is a possibility to delicately  analyze the dynamics of the iterates of the optimization algorithm to reason about exactly which global minimum it converges to. Such types of results are particular challenging because it requires fine-grained control of the optimization dynamics, and rigorous theory can often be  obtained only for relatively simple models such as linear models~\citep{soudry2018implicit} or matrix sensing~\citep{gunasekar2017implicit},  quadratic neural networks~\citep{li2017algorithmic}, a quadratically-parameterized linear model~\citep{woodworth2020kernel,vaskevicius2019implicit,haochen2020shape}, and special cases of two-layer neural nets with relu activations~\citep{li2019towards}.

\vspace{0.03in} \noindent{\bf  Assumptions on data distributions.} The author of the chapter and many others suspect that in the worst case, obtaining the best generalization performance of neural networks may be computationally intractable. Beyond the worst case analysis, people have made stronger assumptions on the data distribution such as Gaussian inputs~\citep{brutzkus2017globally,ge2018learning}, mixture of Gaussians or linearly separable data~\citep{brutzkus2017sgd}. The limitations of making Gaussian assumptions on the inputs are two-fold: a) it's not a realistic assumption; b) it may both over-estimate and under-estimate the difficulties of learning real-world data in different aspects. It is probably not surprising that Gaussian assumption can over-simplify the problem, but there could be other non-Gaussian assumptions that may make the problem even easier than Gaussians (e.g. see the early work in deep learning theory~\citep{arora2014provable}).

%% file: notes.tex
\section{Notes}

\citet{HillarL13} show that a degree four polynomial is NP-hard to optimize  and ~\citet{murty1987some} show that it's also NP-hard to check whether a point is not a local minimum. Our quantitative definition quasi-convexity (Definition~\ref{def:cond}) is from~\citet{hardt2016gradient}. Polyak-Lojasiewicz condition was introduced by~\citet{polyak1963gradient}, and see a recent work of ~\citet{karimi2016linear} for a proof of Theorem~\ref{thm:quasi} . The RSI condition was originally introduced in~\citet{zhang2013gradient}. 

The strict saddle condition was originally defined in~\citep{ge2015escaping}, and we use a variant of the definition formalized in the work of~\citet{lee2016gradient,agarwal2016finding}. Formal versions of Theorem~\ref{thm:global} and Theorem~\ref{thm:localmin} for various concrete algorithms can be found in e.g., ~\citet{nesterov2006cubic, ge2015escaping, agarwal2017finding, carmon2016accelerated,sun2015nonconvex} and their follow-up works.

Theorem~\ref{thm:manifold} is due to \citet[Theorem 12]{2016arXiv160508101B}. We refer the readers to the book~\cite{absil2007optimization} for the definition of gradient and Hessian on the manifolds and for the derivation of equation~\eqref{eqn:mani-gradient} and~\eqref{eqn:mani-hessian}.\footnote{For example, the gradient is defined in~\citet[Section3.6, Equation (3.31)]{absil2007optimization}, and the Hessian is defined in~\citet[Section 5.5, Definition 5.5.1]{absil2007optimization}. ~\cite[Example 5.4.1]{absil2007optimization} gives the Riemannian connection of the sphere $\sphere^{d-1}$ which can be used to compute the Hessian.}

The results covered in Section~\ref{sec:glm} was due to~\cite{kakade2011efficient,hazan2015beyond}. The particular exposition was first written by Yu Bai for the statistical learning theory course at Stanford.  

The analysis of the landscape of the PCA objective was derived in~\citet{baldi1989neural,srebro2003weighted}. The main result covered in Section~\ref{sec:mc} is based on the work~\citep{ge2016matrix}. Please see~\citet{ge2016matrix} for more references on the matrix completion problem. 

Nonconvex optimization has also been used for speeding up convex problems, e.g., the Burer-Monteiro approach~\citep{burer2005local} was theoretically analyzed by the work of~\citet{boumal2016non,bandeira2016low}.

Section~\ref{sec:tensor} is based on the work of \citet{ge2015escaping}. Recently, there have been work on analyzing more sophisticated cases of tensor decomposition, e.g., using Kac-Rice formula~\citep{ge2017optimization} for random over-complete tensors. Please see the reference in~\citet{ge2017optimization} for more references regarding the tensor problems.